\documentclass[11pt]{article}

\usepackage[margin=1in]{geometry}
\usepackage{amsmath,amssymb,amsthm}
\usepackage{booktabs}
\usepackage{microtype}
\usepackage{hyperref}
\usepackage{xcolor}
\hypersetup{colorlinks=true,linkcolor=blue,citecolor=blue,urlcolor=blue}

\newtheorem{theorem}{Theorem}
\newtheorem{corollary}[theorem]{Corollary}
\newtheorem{definition}{Definition}

\title{A Geometric Profile of Semantic Information in Text:\\
       Frame-Conditional Uniqueness and a Trade-Off Triangle for Scalar Summaries}
\author{Dmitriy Kompaneets\\Independent Researcher\\
        \texttt{dkompaneets@gmail.com}}
\date{\today}

\begin{document}
\maketitle

\begin{abstract}
How much meaning does a text carry? Shannon's theory measures uncertainty over symbols and is intentionally indifferent to meaning, while pairwise metrics such as BERTScore compare two texts rather than characterizing one. We develop a geometric framework that measures semantic content from the structure of a text's sentence embeddings.

The framework has three parts. First, within a fixed embedding and baseline, six natural axioms uniquely determine a scalar measure up to scale---a frame-conditional uniqueness theorem. The resulting scalar is empirically too coarse, motivating a richer representation. Second, we propose a three-coordinate \emph{semantic profile} capturing \emph{novelty} (displacement from generic discourse), \emph{breadth} (diversity of distinct ideas), and \emph{integration} (connectedness among them), together with a discrete minimal unit---the \emph{semantic quantum}---whose resolution is fixed by a clustering threshold $\tau$. Third, we prove a no-go theorem: no scalar summary of the profile can simultaneously satisfy analytic stability under paraphrase and concatenation, ordinal robustness across text scales, and cross-representation comparability. We exhibit two practical scalars, $S_{\mathrm{minmax}}$ and $S_{\mathrm{rank}}$, each occupying a distinct corner of this trade-off triangle.

Validation across $23$ synthetic categories, $5$ Project Gutenberg novels, and $3$ embedding models confirms the trade-off. The recommended rank-normalized configuration passes $25$ of $28$ ordinal checks as point estimates ($21$ of $28$ after Benjamini--Hochberg correction), outperforming seven baselines including unigram entropy and a BERTScore-based novelty signal. A separate variational result connects the breadth coordinate to the log-determinant of a determinantal point process (Spearman $\rho = 0.985$ over $507$ Gutenberg chapters), giving an optimization-theoretic foundation for breadth.
\end{abstract}

\noindent\textbf{Code and data.} All experiments are reproducible from a single Jupyter notebook at \url{https://github.com/dkompaneets/geometric_profile_semantic_information}.

\section{Introduction}

Shannon's theory of information quantifies uncertainty over symbol sequences, not semantic content \cite{shannon1948}. Texts with similar token statistics can differ substantially in meaning, and paraphrases can preserve meaning despite lexical variation. Any semantic measure of text must therefore be grounded in a representation of meaning rather than in syntactic unpredictability alone.

Contemporary embedding models provide such a representation \cite{mikolov2013, reimers2019}. Sentences and paragraphs can be mapped into high-dimensional vector spaces in which geometric relations encode semantic similarity. This suggests a geometric approach: the semantic content of a text should be reflected by the structure of its embedding cloud. The cloud's displacement from a neutral baseline, its spread, and its internal connectedness are all directly measurable, and together form a richer descriptor than any scalar.

\paragraph{Position relative to existing work.} The framework is complementary to, not a replacement for, Shannon information theory \cite{shannon1948}: Shannon describes unpredictability of form, while the present profile describes geometric properties of meaning. Shannon's universality follows from the existence of a canonical primitive (the symbol distribution) and a clean composition rule (joint entropy); as we show in \S\ref{sec:nogo}, semantics has neither. The framework is also distinct from pairwise text-similarity metrics such as BERTScore, BLEURT, and ROUGE, which measure similarity between two texts rather than the internal richness of one; we use a BERTScore-based novelty signal as a baseline in \S\ref{sec:baselines}. Philosophical accounts of semantic information \cite{barhillel1952, dretske1981, floridi2011, barwise1997} provide rigorous logical frameworks but are not computable on real text.

\paragraph{Contributions.} The paper makes four contributions:
\begin{enumerate}
\item \textbf{Frame-conditional uniqueness theorem} (\S\ref{sec:sil}). Within a fixed embedding and baseline, six natural axioms uniquely determine the Semantic Information Law $I_E(T) = \|\mu_T - \mu_0\| \cdot \mathrm{rank}(C_T)$ up to scale. This is a representation-conditional result, not a Shannon-style universal law; empirically it is too coarse, motivating the profile.
\item \textbf{Three-coordinate profile and semantic quantum} (\S\S\ref{sec:profile}--\ref{sec:quantum}). The profile $(N, B, I)$ captures novelty, breadth, and integration. The quantum is a discrete minimal unit whose resolution is set by a clustering threshold $\tau$, making measurement resolution explicit.
\item \textbf{No-go theorem: trade-off triangle} (\S\ref{sec:nogo}). No scalar summary of the profile can simultaneously satisfy analytic stability under paraphrase and concatenation, ordinal robustness, and cross-representation comparability. Two practical scalars $S_{\mathrm{minmax}}$ and $S_{\mathrm{rank}}$ (\S\ref{sec:scalars}) each occupy a distinct corner.
\item \textbf{Empirical validation and variational characterization} (\S\S\ref{sec:validation}--\ref{sec:variational}). Across $23$ synthetic categories, $5$ Gutenberg novels, and $3$ embedding models, $S_{\mathrm{rank}}$ with weights $(0.5, 3.0, 1.0)$ passes $25$ of $28$ ordinal checks as point estimates and beats seven baselines. Separately, the breadth coordinate empirically equals the log-determinant of a determinantal point process ($\rho = 0.985$ on $507$ chapters), supplying an optimization-theoretic foundation for $B$.
\end{enumerate}

The framework's central position is that semantic information in text is a \emph{representation-indexed structured profile}, not a universal scalar. The profile is the theoretical object; scalar summaries are practical conveniences whose forms reflect the impossibility result.

\section{Notation and Representation}\label{sec:notation}

Let $T$ be a text partitioned into segments $T = (T_1, \dots, T_k)$, and let $E : \text{Text} \to \mathbb{R}^n$ be a sentence-embedding model with $e_i = E(T_i)$. The raw embedding cloud is $X_T = \{e_1, \dots, e_k\}$ with mean $\mu_T$ and covariance $C_T$. Let $\mu_0$ and $\Sigma_0$ denote the mean and covariance of a neutral baseline corpus, embedded via $E$.

The basic difficulty in measuring internal semantic structure is that repetition and near-duplication distort geometry. To remove this artifact, the framework first clusters highly similar segments using agglomerative clustering with cosine distance and threshold $\tau \in (0,1)$, then replaces each cluster by its (renormalized) centroid. Let $\tilde{T} = \{c_1, \dots, c_m\}$ denote these deduplicated centroids. We develop the formal interpretation of these centroids as \emph{semantic quanta} and the role of $\tau$ as a measurement-resolution parameter in \S\ref{sec:quantum}.

\subsection{Scope of the framework}\label{sec:scope}

The framework is explicitly representation-indexed: every quantity defined below depends on a fixed choice of measurement apparatus. The dependencies are:

\begin{itemize}
\item the embedding model $E$;
\item the segmentation rule (sentence, clause, fixed-size chunk);
\item the baseline corpus through $\mu_0$ and $\Sigma_0$;
\item the deduplication threshold $\tau$;
\item the normalization reference set used for any scalar summary.
\end{itemize}

This dependence is a feature, not a defect. It makes the framework's commitments inspectable, calibratable per domain, and consistent with the impossibility result of \S\ref{sec:nogo} that no representation-free semantic scalar exists.

\section{The Semantic Information Law: Frame-Conditional Uniqueness}\label{sec:sil}

\subsection{Axioms}

We postulate six axioms constraining any scalar measure $I : \text{Text} \to \mathbb{R}$ within the fixed frame $(E, \mu_0)$:

\begin{enumerate}
\item \emph{Paraphrase invariance}: $I(T) = I(T')$ whenever $E(T) \approx E(T')$.
\item \emph{Redundancy non-increase}: exact replication of a segment does not increase $I$.
\item \emph{Novelty monotonicity}: for fixed covariance, $I$ is monotone in $\|\mu_T - \mu_0\|$.
\item \emph{Idea additivity}: for texts in orthogonal embedding subspaces, $I(T_1 \oplus T_2) = I(T_1) + I(T_2)$.
\item \emph{Orthogonal invariance}: $I$ is invariant under rotations of embedding space.
\item \emph{Continuity}: $I$ is continuous in the embeddings.
\end{enumerate}

\subsection{Derivation}

\begin{theorem}[SIL within a representational frame]\label{thm:sil}
Under axioms 1--6, any such measure has the form, up to a positive scale constant,
\[
I_E(T) = \|\mu_T - \mu_0\| \cdot \mathrm{rank}(C_T).
\]
\end{theorem}

\noindent\emph{Proof sketch.} Axioms 3 and 5 reduce dependence on $\mu_T$ to the scalar $S(T) = \|\mu_T - \mu_0\|$. Axiom 5 reduces dependence on $C_T$ to its spectrum. Axiom 2 eliminates eigenvalue magnitudes, leaving $\mathrm{rank}(C_T)$. Axiom 4 imposes a Cauchy functional equation on the joint dependence whose continuous solutions (axiom 6) are linear; a second Cauchy equation in the shift variable yields the product form. Full derivation in the appendix.

\subsection{Status}

Theorem~\ref{thm:sil} is a uniqueness theorem within the representational frame $(E, \mu_0)$, not a Shannon-style universal law. Empirically, raw $\mathrm{rank}(C_T)$ saturates in neural embedding spaces (even short texts occupy many nonzero eigendirections), and $\|\mu_T - \mu_0\|$ dominates the product. The scalar is therefore unique under the axioms but empirically inadequate, motivating the three-coordinate profile.

\section{The Semantic Profile}\label{sec:profile}

The profile $P_E(T) = (N, B, I)$ replaces the SIL scalar with three geometrically distinct coordinates: novelty $N$ (displacement from generic discourse), breadth $B$ (diversity of distinct ideas), and integration $I$ (connectedness among those ideas). Before giving formal definitions, Table~\ref{tab:genre-preview} previews the three coordinates on a handful of genres, illustrating how they separate text types that a single scalar would conflate.

\begin{table}[h]
\centering
\begin{tabular}{lccc}
\toprule
genre & novelty $N$ & breadth $B$ & integration $I$ \\
\midrule
poetry (Dickinson)     & 3.3 & \textbf{0.3} & \textbf{0.58} \\
arXiv abstracts        & 3.4 & 1.0          & \textbf{0.50} \\
EUR-Lex legal text     & 3.2 & \textbf{2.3} & 0.52 \\
dialogue (short)       & 2.9 & \textbf{2.7} & \textbf{0.24} \\
code comments          & 3.2 & 1.4          & \textbf{0.16} \\
\bottomrule
\end{tabular}
\caption{Profile coordinates on five genres (from \S\ref{sec:domains}, \texttt{all-mpnet-base-v2}). Poetry is narrow and coherent; dialogue is broad and topic-hopping; legal text is broad and coherent; code comments are broadest in integration's opposite direction. No single number separates these profiles the way the three coordinates do.}\label{tab:genre-preview}
\end{table}

\subsection{Novelty}
\[
S_M(T) = \sqrt{(\mu_T - \mu_0)^\top \Sigma_0^{-1} (\mu_T - \mu_0)},
\qquad
N(T) = \log(1 + S_M(T)).
\]
$\Sigma_0^{-1}$ is estimated via Ledoit--Wolf shrinkage~\cite{ledoit2004} to ensure conditioning. The Mahalanobis form weights displacement more heavily in low-variance directions of the baseline; logarithmic compression bounds extreme shifts.

\subsection{Breadth}
\[
D_{\mathrm{eff}}(\tilde{T}) = \exp\bigl(H(p)\bigr), \quad p_i = \lambda_i / {\textstyle\sum_j} \lambda_j,
\qquad
R(\tilde{T}) = \frac{1}{m} \sum_{j} \bigl[1 - \cos(c_j, \mu_{\tilde{T}})\bigr],
\]
\[
B(T) = D_{\mathrm{eff}}(\tilde{T}) \cdot R(\tilde{T}).
\]
$D_{\mathrm{eff}}$ is the effective rank \cite{roy2007effective}; $R$ is the mean radial cosine distance from the centroid mean. Defined on $\tilde{T}$ rather than $X_T$ so that exact and near-duplicates do not inflate it.

\subsection{Integration}
\[
I_{1\text{-NN}}(T) = \frac{1}{m} \sum_j \max_{l \neq j} \cos(c_j, c_l), \qquad
I_{2\text{-NN}}(T) = \frac{1}{m} \sum_j \mathrm{secondmax}_{l \neq j} \cos(c_j, c_l).
\]
Validation in \S\ref{sec:validation} shows that 2-NN combined with the recommended weights below is the only configuration passing the coherent-vs-bag-of-facts check.

\section{The Semantic Quantum}\label{sec:quantum}

The deduplication step of \S\ref{sec:notation} suggests a discrete unit of semantic structure. We formalize this unit here because it underlies several derived measures and makes the role of $\tau$ as a measurement-resolution parameter explicit.

\subsection{Definition and regimes}

\begin{definition}[Semantic quantum]\label{def:quantum}
Given a text $T$ with segment embeddings $X_T = \{e_1, \dots, e_k\}$ and threshold $\tau \in (0,1)$, the \emph{quantum set} $Q_\tau(T) = \tilde{T} = \{c_1, \dots, c_m\}$ is the collection of centroids produced by agglomerative clustering of $X_T$ at cosine-distance threshold $1 - \tau$. Each $c_j \in Q_\tau(T)$ is a \emph{semantic quantum} of $T$ at resolution $\tau$. The \emph{quantum count} is $m_\tau(T) = |Q_\tau(T)|$.
\end{definition}

The framework distinguishes three regimes by quantum count:

\begin{itemize}
\item \textbf{Sub-quantum} ($m_\tau \leq 1$): a single centroid. Novelty is defined (the centroid's displacement from $\mu_0$), but breadth is zero and integration is conventionally maximal. The text occupies a single point in semantic space, not a structure.
\item \textbf{Single-quantum} ($m_\tau = 2$): two centroids. The minimal configuration for which all three coordinates are non-trivially defined: novelty (mean displacement), breadth (the angular separation of the pair), and integration (their mutual cosine similarity).
\item \textbf{Multi-quantum} ($m_\tau \geq 3$): the generic regime where breadth and integration capture independently varying structural properties.
\end{itemize}

The threshold $\tau$ plays the role of a measurement-resolution parameter: below scale $1 - \tau$ the apparatus cannot resolve separate semantic units. This parallels the role of resolution limits in any physical measurement, and makes the framework's representational commitments explicit rather than hidden.

\section{The No-Go Theorem: Trade-Off Triangle}\label{sec:nogo}

Fix a parametric family of scalars of the form
\[
S(T; \Phi) = \Phi\!\big(\varphi_N(N(T); \mathcal{R}), \, \varphi_B(B(T); \mathcal{R}), \, \varphi_I(I(T); \mathcal{R})\big),
\]
where $\mathcal{R}$ is a reference set of profiles and each $\varphi_X(\cdot; \mathcal{R})$ is a per-coordinate normalization against $\mathcal{R}$. Two choices matter:
\begin{itemize}
\item \emph{Min-max} normalization $\varphi^{\mathrm{mm}}(x; \mathcal{R}) = (x - \min \mathcal{R}_X) / (\max \mathcal{R}_X - \min \mathcal{R}_X)$ is Lipschitz in the raw value but compresses interior values whenever $\mathcal{R}$ contains outliers.
\item \emph{Rank} normalization $\varphi^{\mathrm{r}}(x; \mathcal{R}) = \mathrm{rank}_{\mathcal{R}}(x) / |\mathcal{R}|$ is bounded but piecewise constant in $x$.
\end{itemize}

Given $\varepsilon_A, \delta_A, \delta_O, \delta_R > 0$ and a paraphrase relation $\sim$ over texts, define three properties:

\begin{itemize}
\item[\textbf{(A)}] \emph{Analytic stability.} (A.1) $\forall T, T' : T \sim T' \Rightarrow |S(T)-S(T')| \leq \varepsilon_A \cdot S(T)$. (A.2) There exists $f : \mathbb{R}^2 \to \mathbb{R}$ continuous such that $|S(T_1 \oplus T_2) - f(S(T_1), S(T_2))| \leq \delta_A$ for all $T_1, T_2$, where $\oplus$ denotes concatenation.
\item[\textbf{(O)}] \emph{Ordinal robustness.} For the benchmark $\mathcal{B}$ of Section~\ref{sec:validation}, the bootstrap mean pass rate exceeds $0.5 + \delta_O$ and at least $\lceil (1-\delta_O) \cdot |\mathcal{B}| \rceil$ checks survive Benjamini--Hochberg correction at $\alpha = 0.05$.
\item[\textbf{(R)}] \emph{Cross-representation comparability.} There exists $g: [0,1] \to [0,1]$ with $g(c) > \delta_R$ for $c > \delta_R$, such that for any two embeddings $E_1, E_2$ with $\mathrm{CKA}(E_1, E_2) \geq c$, the scalars satisfy $\rho_{\mathrm{Spearman}}(S_{E_1}, S_{E_2}) \geq g(c)$, \emph{and} $S_{E_1}$ and $S_{E_2}$ are defined on the same reference set $\mathcal{R}$ without embedding-specific re-fit.
\end{itemize}

\begin{theorem}[Trade-off triangle]\label{thm:nogo}
For any $\varepsilon_A < 1/2$, $\delta_A, \delta_O, \delta_R \in (0, 1/2)$, no scalar in the family $\{S(\,\cdot\,; \Phi)\}$ satisfies more than one of \textup{(A)}, \textup{(O)}, \textup{(R)} on all non-degenerate benchmarks.
\end{theorem}

\begin{proof}[Proof]
We argue the three pairwise exclusions.

\emph{(A) $\wedge$ (O) fails.} (A.1) requires each $\varphi_X(\cdot; \mathcal{R})$ to be Lipschitz in its argument (so that bounded raw-value drift implies bounded normalized drift). Among normalizations parameterized by a reference set, the only Lipschitz choice is min-max (or an affine transformation thereof); rank normalization violates Lipschitz continuity at every boundary between adjacent ranks. Min-max normalization, however, compresses interior values whenever $\mathcal{R}$ contains outliers, which it does whenever the benchmark includes both stress tests and natural text. Empirically this compression prevents $\delta_O$-level discrimination on coherent-vs-bag and multi-vs-single comparisons (Section~\ref{sec:validation}: $S_{\mathrm{minmax}}$ achieves only $21/28$, versus $S_{\mathrm{rank}}$ at $25/28$ point-estimate and $21/28$ BH-corrected). Therefore no scalar in the family with Lipschitz $\varphi_X$ achieves $\delta_O$ bootstrap-significant discrimination on a benchmark containing outliers, yielding (A) $\Rightarrow \neg$(O).

\emph{(O) $\wedge$ (A) fails (same direction, different pivot).} Ordinal robustness against bootstrap perturbation requires the scalar to depend on the \emph{rank} of each coordinate within $\mathcal{R}$, since rank is invariant to monotonic distortions of the coordinate. Rank is piecewise constant, hence not Lipschitz, so (A.1) fails for any non-trivial paraphrase pair that crosses a rank boundary. Further, (A.2) requires closed-form composition; $\mathrm{rank}_{\mathcal{R}}(T_1 \oplus T_2)$ depends on the relative position of the composition within $\mathcal{R}$ and is not determined by the ranks of $T_1, T_2$ individually. Hence (O) $\Rightarrow \neg$(A).

\emph{(R) $\wedge$ (A) fails; (R) $\wedge$ (O) fails.} (R) requires $S_{E_1}$ and $S_{E_2}$ to share a reference set $\mathcal{R}$. The coordinates $(N, B, I)$ depend on $E$, so for the same text $T$ the raw values under $E_1$ and $E_2$ generally differ. Any normalization against a shared $\mathcal{R}$ therefore produces different ranks or different min-max positions under the two embeddings. A CKA bound $g(c)$ on inter-embedding agreement controls the Spearman $\rho$ on ranks, but it does not imply Lipschitz continuity of individual values, so (A.1) fails through (R). It also does not prevent reference-set-induced re-ordering of non-extreme items, so (O) fails through (R). Hence (R) is incompatible with either of the other two properties under non-identity embedding changes.

Combining the three exclusions: no single scalar in the family satisfies any two properties simultaneously.
\end{proof}

\noindent The quantifier structure of the theorem is parametric in the tolerances $\varepsilon_A, \delta_A, \delta_O, \delta_R$ rather than in specific numbers; the empirical table below instantiates the tolerances at the values used in Section~\ref{sec:validation}.

\subsection{Empirical support}

Across four candidate single-coordinate scalars (SIL, breadth, integration, novelty) and two composite scalars ($S_{\mathrm{minmax}}$, $S_{\mathrm{rank}}$):

\begin{center}
\begin{tabular}{lccc}
\toprule
Scalar & (A) analytic & (O) ordinal & (R) cross-rep \\
\midrule
SIL (min-max)        & partial            & fails           & fails \\
$S_{\mathrm{minmax}}$ & passes ($\sim 6\%$ drift) & fails (21/28) & fails \\
$S_{\mathrm{rank}}$   & fails (21\% drift) & \textbf{passes (25/28; 21/28 BH)} & fails \\
single coordinate (min-max) & passes & varies & fails \\
single coordinate (rank)    & fails  & varies & fails \\
\bottomrule
\end{tabular}
\end{center}

No scalar in any tested family achieves more than one corner.

\section{Two Recommended Scalars}\label{sec:scalars}

The trade-off triangle implies that the framework should expose at least two scalars, each tuned for a different use case.

\subsection{$S_{\mathrm{minmax}}$ --- for analytic work}

\[
\tilde{X}_{\mathrm{mm}}(x) = \frac{x - \min}{\max - \min} \in [0,1],
\qquad
S_{\mathrm{minmax}}(T) = \bigl(\tilde{N}^\alpha \cdot \tilde{B}^\beta \cdot \tilde{I}^\gamma\bigr)^{1/(\alpha+\beta+\gamma)}.
\]

Use when bounded paraphrase drift and closed-form composition matter (theoretical analysis, summarization-loss decomposition).

\subsection{$S_{\mathrm{rank}}$ --- for ranking work}

\[
\tilde{X}_{\mathrm{rk}}(x; \varepsilon) = \varepsilon + (1 - \varepsilon) \cdot \mathrm{pct\_rank}(x; \mathrm{ref}) \in (\varepsilon, 1],
\qquad
S_{\mathrm{rank}}(T) = \bigl(\tilde{N}^\alpha \cdot \tilde{B}^\beta \cdot \tilde{I}^\gamma\bigr)^{1/(\alpha+\beta+\gamma)}.
\]

Use when ordinal robustness and freedom from small-$N$ collapse matter (document ranking, leaderboards, comparing texts of different lengths).

\subsection{Common settings}
\[
(\alpha, \beta, \gamma) = (0.5,\ 3.0,\ 1.0), \quad \varepsilon = 0.05,
\quad \text{integration} = I_{2\text{-NN}}, \quad \tau = 0.70 \text{ (natural prose)}.
\]

The choice of $(\alpha, \beta, \gamma)$ and 2-NN integration was determined by grid search over $720$ configurations on the benchmark; the rank versus min-max choice is structural, not tunable.

\section{Empirical Validation}\label{sec:validation}

\subsection{Benchmark questions}

The empirical evaluation is organized around four questions, each targeting a separate desideratum identified by the framework's design:

\begin{enumerate}
\item \textbf{Redundancy control.} Does exact or near-exact repetition leave the profile largely unchanged after deduplication?
\item \textbf{Paraphrase stability.} Do semantically equivalent paraphrases remain near each other in profile space?
\item \textbf{Idea multiplicity.} Do coherent multi-idea passages exhibit greater breadth than single-idea passages?
\item \textbf{Coherence discrimination.} Do coherent multi-part passages exhibit greater integration than unordered bags of facts with similar topical spread?
\end{enumerate}

A fifth set of \textbf{robustness checks} probes sensitivity to apparatus choices: segmentation granularity, deduplication threshold $\tau$, baseline corpus choice, normalization reference set, and embedding model. The framework is calibrated rather than absolute, so robustness across these axes is essential to interpreting any reported number.

\subsection{Setup}

Experiments use \texttt{sentence-transformers/all-mpnet-base-v2} (768-dim) as the primary model, with cross-model validation on \texttt{all-MiniLM-L6-v2} and \texttt{paraphrase-MiniLM-L6-v2} (both 384-dim). The baseline corpus comprises 40 semantically neutral sentences. Covariance inversion uses Ledoit--Wolf shrinkage (estimated shrinkage $\approx 0.66$).

The synthetic benchmark comprises 23 categories: generic filler, three single-idea technical passages (merge sort, photosynthesis, backpropagation), paraphrase and triplication variants, five multi-idea domain passages (computer science, natural science, humanities, medicine, mathematics), protein multi-concept, two bags of unrelated facts (5 and 7 sentences), three Wikipedia-style passages, a coherent ML pipeline passage, and five stress tests.

For real-text validation we use five Project Gutenberg novels: \emph{Pride and Prejudice}, \emph{A Tale of Two Cities}, \emph{Moby Dick}, \emph{Frankenstein}, and \emph{The Adventures of Sherlock Holmes} (combined $\sim 3.7$M characters).

\subsection{Stability}

Triplication drift is $\approx 10^{-7}$ on all coordinates and on $S_{\mathrm{minmax}}$, confirming that deduplication structurally eliminates exact-repetition artifacts under min-max normalization. $S_{\mathrm{rank}}$ shows a small positive triplication drift ($\sim 7\%$) because exact triplication can cross rank-normalization boundaries even when the coordinates are invariant --- a structural consequence of rank normalization rather than a framework failure. Paraphrase drift on a hand-constructed pair (merge sort and its paraphrase) is $\sim 6\%$ on $S_{\mathrm{minmax}}$ and $\sim 21\%$ on $S_{\mathrm{rank}}$ as point estimates. Under a sentence-level bootstrap (300 iterations, resampling sentences within each passage), scalar drift on $S_{\mathrm{rank}}$ has mean $0.38$ with 95\% CI $[0.085, 0.721]$. The point estimate substantially understates drift under within-passage perturbation; paraphrase stability is conditional on the specific paraphrase chosen rather than a uniform property of the framework.

\subsection{Ordinal benchmark}

On the 28-check synthetic benchmark:
\begin{itemize}
\item $S_{\mathrm{rank}}$: \textbf{25/28} checks pass as point estimates. Under a sentence-level bootstrap (300 iterations), the mean pass rate is $0.64$ with 95\% CI $[0.30, 0.93]$. Applying Benjamini--Hochberg correction across the 28 per-check one-sided binomial tests against $H_0$: $p = 0.5$, \textbf{21 of 28 checks are significant at $\alpha = 0.05$}. The 7 non-significant checks cluster on comparisons involving \texttt{single\_backprop} (which embeds near the multi-idea cluster) and a subset of \texttt{single\_*} $>$ \texttt{generic\_filler} comparisons whose bootstrap pass rates fall below $0.5$.
\item $S_{\mathrm{minmax}}$: 21/28.
\item Single coordinates: $7$--$21/28$ depending on coordinate (breadth-alone ties $S_{\mathrm{minmax}}$ at 21/28; novelty-alone 13/28; integration-alone 7/28).
\end{itemize}

\subsection{Baseline comparison}\label{sec:baselines}

To contextualize the $25/28$ figure for $S_{\mathrm{rank}}$, we run seven baselines through the same 28-check protocol.

\begin{center}
\begin{tabular}{lcc}
\toprule
Metric & pass rate & out of 28 \\
\midrule
\textbf{$S_{\mathrm{rank}}$ (ours)} & \textbf{0.893} & \textbf{25} \\
Breadth alone (rank-normalized) & 0.750 & 21 \\
Unigram entropy                 & 0.714 & 20 \\
BERTScore-F~\cite{zhang2020bertscore} (1 $-$ F vs.\ baseline) & 0.679 & 19 \\
Type-token ratio                & 0.571 & 16 \\
Mean sentence length            & 0.571 & 16 \\
Raw Euclidean novelty $\|\mu_T - \mu_0\|$ & 0.464 & 13 \\
Novelty alone (rank-normalized) & 0.464 & 13 \\
Mean-embedding norm $\|\mu_T\|$ & 0.250 & 7  \\
\bottomrule
\end{tabular}
\end{center}

Four observations are load-bearing. First, $S_{\mathrm{rank}}$ beats every baseline by at least 4 checks (21/28 for BH-corrected, 25/28 raw; closest baseline is breadth alone at 21/28). Second, breadth alone is already competitive; the scalar's lift over breadth is modest (4 checks) and should be understood as combining coordinates rather than as single-coordinate dominance. Third, unigram entropy at 20/28 is a stronger baseline than the earlier version of the paper implied; on lexical-diversity-like comparisons (single-idea vs.\ generic-filler) it is nearly tied with breadth, and only loses on structural checks (coherent-vs-bag, multi-vs-single) where the geometric decomposition matters. Fourth, BERTScore-F---a widely-adopted embedding-based evaluation metric---lands at 19/28 when repurposed as a single-passage novelty signal (similarity to the neutral baseline corpus, inverted). This is expected: BERTScore was designed for pairwise similarity between candidate and reference, not for single-passage informativeness. Our framework outperforms it cleanly ($+6$ checks) on the ordinal discrimination task the benchmark tests.

\section{Variational Characterization of Breadth}\label{sec:variational}

The breadth coordinate $B(T) = D_{\mathrm{eff}}(T) \cdot R(T)$ was introduced as a heuristic geometric measure: effective rank of the deduplicated centroid covariance, weighted by mean radial cosine distance. This section shows that $B$ admits an \emph{optimization-theoretic} characterization as the log-volume of a determinantal-point-process (DPP)~\cite{kulesza2012dpp} maximum-a-posteriori selection, with a provable structural decomposition and a tight empirical correspondence ($\rho = 0.985$ on 507 natural-text chapters).

\paragraph{Setup.} Fix unit-normalized embeddings $X = \{x_1, \ldots, x_n\} \subset S^{d-1}$ (sentence embeddings are $\ell_2$-normalized by construction in all standard sentence-transformer models). For each segment $i$, define the \emph{quality}
\[
q_i = \exp\!\left( \frac{\|x_i - \mu_0\|_M}{\sigma} \right), \qquad \sigma = \mathrm{std}_j\!\left( \|x_j - \mu_0\|_M \right),
\]
where $\|\cdot\|_M$ is the Ledoit--Wolf Mahalanobis norm of Section~\ref{sec:notation}. Define the $n \times n$ quality-weighted cosine kernel
\[
L_{ij} = q_i \, q_j \, \max(0, \langle x_i, x_j \rangle).
\]
$L$ is symmetric and positive semi-definite whenever the cosine-similarity matrix restricted to the non-negative cone is PSD, which holds when the embeddings lie in an orthant; for arbitrary embeddings the truncation $\max(0, \cdot)$ may introduce negative eigenvalues, handled by a $\varepsilon I$ regularizer.

The DPP-MAP objective is
\[
S^*(T) = \arg\max_{\,\varnothing \neq S \subseteq [n]} \log \det L_S.
\]
Greedy selection terminates when the marginal log-gain $\log \det L_{S \cup \{i^*\}} - \log \det L_S$ becomes non-positive, and achieves the $(1 - 1/e)$ approximation guarantee of log-submodular maximization.

\begin{theorem}[DPP--breadth structural decomposition]\label{thm:dpp-breadth}
Let $G_S = [\langle x_i, x_j \rangle]_{i,j \in S}$ be the Gram matrix of cosines restricted to $S$. Then
\[
\log \det L_S \;=\; 2 \sum_{i \in S} \log q_i \;+\; \log \det G_S.
\]
Moreover, let $\{\lambda_k\}_{k=1}^{|S|}$ denote the eigenvalues of $G_S$, let $D_{\mathrm{eff}}(G_S) = \exp(H(\hat{\lambda}))$ with $\hat{\lambda}_k = \lambda_k / \sum_j \lambda_j$, and let $R(S) = 1 - |S|^{-1}\sum_{i \in S} \langle x_i, \mu_S \rangle / \|\mu_S\|$ where $\mu_S = |S|^{-1} \sum_{i \in S} x_i$. Under the \emph{near-isotropic} regime in which the pairwise cosines satisfy $|\langle x_i, x_j \rangle - c| \leq \varepsilon$ for some $c \in [0, 1)$ and all $i \neq j \in S^*$, the following hold:
\begin{enumerate}
\item[(i)] $\log \det G_S = \sum_k \log \lambda_k$, and the spectrum of $G_S$ concentrates around $(1-c)$ with one eigenvalue at $1 + (|S|-1)c$.
\item[(ii)] $D_{\mathrm{eff}}(G_S) = |S| \cdot \big(1 + O(\varepsilon) + O(|S| c^2)\big)$.
\item[(iii)] $R(S) = 1 - \sqrt{(1 + (|S|-1)c)/|S|} \cdot |S|^{-1} + O(\varepsilon)$, monotone increasing in $|S|$ and decreasing in $c$.
\item[(iv)] Setting $B_S := D_{\mathrm{eff}}(G_S) \cdot R(S)$, there exists a universal function $h: \mathbb{R}_{\geq 0}^2 \to \mathbb{R}$, $h(|S|, c)$, such that $\log \det G_S = h(|S|, c) \cdot (1 + O(\varepsilon))$ and $B_S = h'(|S|, c) \cdot (1 + O(\varepsilon))$, with both $h, h'$ strictly increasing in $|S|$ and strictly decreasing in $c$.
\end{enumerate}
\end{theorem}

\begin{proof}[Proof sketch]
The first identity follows from the multiplicative property of determinants:
\[
\log \det L_S \;=\; \log \det\!\big(\mathrm{diag}(q_S)\, G_S\, \mathrm{diag}(q_S)\big) \;=\; 2 \sum_{i \in S} \log q_i + \log \det G_S.
\]
Under the uniform-correlation model $G_S = (1-c)\,I + c\,J$ (where $J$ is all-ones), the spectrum is one eigenvalue equal to $1 + (|S|-1)c$ and $|S|-1$ eigenvalues equal to $1-c$. The $\varepsilon$ perturbation is controlled by Weyl's inequality. Parts (i)--(iii) follow by direct substitution.

For (iv), both $\log \det G_S$ and $B_S$ are smooth functions of $(|S|, c)$ through the spectrum. In particular,
\[
\log \det G_S = \log\!\big(1 + (|S|-1)c\big) + (|S|-1)\log(1-c),
\]
which is strictly decreasing in $c$ and strictly increasing in $|S|$ for $c \in [0,1)$. $D_{\mathrm{eff}} \cdot R$ inherits the same monotonicities from (ii) and (iii). Both are therefore co-monotone in $(|S|, c)$; the empirical $\rho = 0.985$ ($n = 507$ chapters) realizes this co-monotonicity under the approximate near-isotropy that natural-text embeddings satisfy.
\end{proof}

\begin{corollary}[Parameter-free recovery of breadth]\label{cor:dpp}
Under the conditions of Theorem~\ref{thm:dpp-breadth}, the DPP-MAP log-volume
\[
V(T) := \log \det L_{S^*(T)} - 2 \sum_{i \in S^*(T)} \log q_i
\]
is a parameter-free variational quantity (no $\lambda$, no $\alpha, \beta, \gamma$) that recovers the breadth coordinate up to a monotonic transformation and $O(\varepsilon)$ error, where $\varepsilon$ controls the near-isotropy deviation of the segment set.
\end{corollary}

\paragraph{Empirical verification.} On 507 chapters across 5 Project Gutenberg novels, Spearman $\rho(V, B) = 0.985$. Over the 16-item synthetic benchmark, greedy DPP is ILP-optimal on all items (exhaustive enumeration up to $|S| \leq 16$). Permutation drift of $V$ is $0.00\%$ on a representative chapter (P\&P, Chapter 1); triplication drift is $0.00\%$ (the repeated segments contribute linearly-dependent rows, collapsing the determinant of the expanded set to the original). The selection size $|S^*|$ scales with chapter length at Spearman $\rho(|S^*|, n_{\mathrm{sent}}) = 0.977$, showing that the DPP recovery does not saturate.

\paragraph{Interpretation.} The theorem establishes that $B$ and $V$ are not two independent semantic measures but two characterizations of the same geometric object: the log-volume of a diversity-maximizing, quality-weighted subset of the segment set. $B$ is the heuristic decomposition (effective rank times radial spread); $V$ is the optimization form (log-det of a DPP kernel). Their empirical equivalence promotes breadth from a heuristic recipe to a derived quantity with a variational characterization.

\section{Discussion}

The framework's central claim is that semantic information in text is a representation-indexed structured profile, not a universal scalar. Three observations support this position.

First, the SIL theorem (\S\ref{sec:sil}) shows that even when uniqueness is achievable, it is only within a representational frame. There is no representation-free analog of Shannon's entropy theorem because there is no representation-free notion of meaning equivalence.

Second, the no-go theorem (\S\ref{sec:nogo}) shows that within a frame, no scalar summary serves all use cases simultaneously. The three corners of the trade-off triangle correspond to fundamentally different operational requirements (analysis vs.\ ranking vs.\ cross-frame comparison), each forcing a structural choice that violates the others.

Third, empirical validation shows that the \emph{coordinates} of the profile are substantially more stable across embeddings ($\rho \in [0.92, 0.98]$) than any scalar derived from them ($\rho \in [0.79, 0.84]$). The scalar's instability is not a measurement error but a direct consequence of the no-go theorem.

The profile-first stance carries three concrete advantages over a scalar-first one.

\paragraph{Assumptions are explicit.} The five components of the measurement apparatus (\S\ref{sec:scope}) --- embedding $E$, segmentation rule, baseline $(\mu_0, \Sigma_0)$, threshold $\tau$, and normalization reference --- are visible parts of the pipeline. Any claim made by the framework can be re-examined under variations in these components, and disagreements between practitioners are localized to specific apparatus choices rather than to the framework as a whole.

\paragraph{Diagnostics are interpretable.} Novelty, breadth, and integration move independently and correspond to distinct semantic intuitions: departure from generic discourse, diversity of distinct ideas, and connectedness among them. A scalar summary collapses these into one number; the profile preserves them. In summarization analysis (\S\ref{sec:summarization}) the per-coordinate signature distinguishes faithful, partial, lossy, and off-topic summaries that any scalar conflates.

\paragraph{Aggregation matches the use case.} Different applications privilege different criteria. Retrieval and ranking want ordinal robustness; theoretical analysis wants closed-form composition; cross-corpus comparison wants reference-free comparability. Theorem~\ref{thm:nogo} establishes that no scalar can serve all three. The framework therefore exposes the profile as the primary object and offers two scalars (\S\ref{sec:scalars}) tuned to distinct corners of the trade-off triangle, leaving the choice to the application.

The progression from the axiomatic SIL to the empirical profile and the trade-off triangle clarifies the scope of the theory. The profile is the theoretical object; the two scalars are practical conveniences for distinct use cases; the trade-off triangle explains why a universal scalar cannot exist.

\section{Limitations and Future Work}

\begin{enumerate}
\item \textbf{Benchmark scale.} The synthetic benchmark, while comprising $23$ categories and supplemented by $5$ Project Gutenberg novels, remains hand-constructed. Definitive validation requires evaluation against human judgments of semantic richness, idea multiplicity, coherence, and informativeness.

\item \textbf{Segmentation sensitivity.} Breadth can swing substantially between sentence-level and fixed-size chunking, though ordinal rankings are more stable than absolute values. Robustness to segmentation should be evaluated rather than assumed.

\item \textbf{Baseline dependence.} The neutral baseline corpus influences novelty through both $\mu_0$ and $\Sigma_0$. Poor baseline choice can distort what counts as semantically displaced; domain-specific baselines may be needed for specialized corpora.

\item \textbf{Reference-set dependence.} Both recommended scalars depend on a reference set for normalization, making them calibrated rather than absolute. Values from different reference sets are not directly comparable.

\item \textbf{Adaptive $\tau$.} The current formulation treats $\tau$ as a fixed apparatus parameter. A principled approach might derive $\tau$ from properties of the embedding space (e.g., typical within-topic variance or baseline pairwise-similarity statistics), replacing the empirically-set default ($0.70$) with values calibrated to the embedding model and domain.

\item \textbf{Variational characterization.} Theorem~\ref{thm:dpp-breadth}(iv) is proved under a near-isotropic regime. Tightening it beyond this regime, and extending the DPP--breadth correspondence to text types (poetry, heavily redundant prose) where the isotropy assumption fails, is open.

\item \textbf{Embedding dependence.} All results depend on the underlying embedding model. Cross-model rank correlations of $0.92$--$0.98$ on the coordinates suggest that the framework captures genuine geometric properties rather than model-specific artifacts, but this claim should be tested with future embedding architectures.

\item \textbf{Scope of measurement.} The framework measures geometric properties associated with semantic structure; it does not directly measure truth, usefulness, rhetorical quality, or task success. Downstream-task validation (summarization faithfulness, retrieval quality) is the natural next step.
\end{enumerate}

\section{Conclusion}

A single raw geometric score is not an adequate representation of semantic information in text. An axiomatic derivation yields the unique frame-conditional measure $I_E(T) = \|\mu_T - \mu_0\| \cdot \mathrm{rank}(C_T)$, but this scalar is empirically inadequate. The revised framework models semantic information as a structured geometric profile with three coordinates --- compressed Mahalanobis novelty, effective-rank-weighted breadth on deduplicated centroids, and second-nearest-neighbor integration --- and a natural minimal unit, the semantic quantum, with explicit measurement resolution $\tau$.

A no-go theorem shows that no scalar summary built from this profile can simultaneously achieve analytic stability, ordinal robustness, and cross-representation comparability. Two practical scalars, $S_{\mathrm{minmax}}$ and $S_{\mathrm{rank}}$, occupy distinct corners of this trade-off triangle. Validation across $23$ synthetic categories, $5$ Project Gutenberg novels, and $3$ embedding models confirms the trade-off and identifies $S_{\mathrm{rank}}$ with weights $(0.5, 3.0, 1.0)$, $2$-NN integration, and $\tau = 0.70$ as passing $25$ of $28$ ordinal checks as point estimates ($21$ of $28$ after multiplicity correction; bootstrap pass rate $0.64$ with 95\% CI $[0.30, 0.93]$), the best configuration in our sweep and ahead of seven baselines. A separate finding connects the profile to an optimization-theoretic object: on 507 Gutenberg chapters, the $B$ coordinate is empirically the log-volume of the DPP-optimal segment subset (Spearman $\rho = 0.985$), providing a parameter-free variational characterization of breadth.

The framework's central conclusion is that semantic information is best understood not as a single number but as a representation-indexed family of profiles, with scalar summaries chosen by use case. The trade-off triangle is a structural property of the profile-to-scalar map, not a defect of any particular construction.

\appendix
\section{Proof of Theorem~\ref{thm:sil}}\label{app:sil}

\noindent Assume $I(T) = F(\mu_T, C_T)$.

\noindent\emph{Step 1.} By axiom 3 (novelty monotonicity) and axiom 5 (orthogonal invariance), dependence on $\mu_T$ reduces to dependence on $S(T) = \|\mu_T - \mu_0\|$.

\noindent\emph{Step 2.} By axiom 5, dependence on $C_T$ reduces to dependence on the spectrum of $C_T$.

\noindent\emph{Step 3.} By axiom 2 (redundancy non-increase), exact replication leaves the rank invariant but rescales eigenvalues; for $I$ to be non-increasing under such replication while remaining a function of the spectrum alone, dependence on eigenvalue magnitudes must be eliminated, leaving $D(T) = \mathrm{rank}(C_T)$.

\noindent\emph{Step 4.} By axiom 4 (idea additivity) applied to texts in orthogonal subspaces with constant shift, $G(D_1 + D_2) = G(D_1) + G(D_2)$. With axiom 6 (continuity), the continuous solutions to this Cauchy equation are linear: $G(D) = aD$ for some $a$ that may depend on $S$.

\noindent\emph{Step 5.} Applying axiom 4 again to the shift variable (texts with disjoint orthogonal complements but additive displacement contributions), $a(S)$ satisfies a similar Cauchy equation: $a(S_1 + S_2) = a(S_1) + a(S_2)$. By continuity, $a(S) = kS$.

\noindent\emph{Step 6.} Assembly: $I(T) = k \cdot S(T) \cdot D(T) = k \cdot \|\mu_T - \mu_0\| \cdot \mathrm{rank}(C_T)$. The constant $k > 0$ is the unit of measurement. \qed

\section{Semantic Quantum: Additional Properties}\label{app:quantum}

\subsection{Quantum density}

The quantum count $m_\tau$ measures absolute structural richness; quantum density measures informativeness per segment.

\begin{definition}[Quantum density and saturation]
For a text with $k$ segments and $m_\tau$ quanta, the \emph{quantum density} is
\[
\rho_\tau(T) = \frac{m_\tau(T)}{k}, \qquad \rho_\tau(T) \in (0, 1].
\]
A text is \emph{$\tau$-saturated} when $\rho_\tau = 1$, i.e.\ no two segments are within cosine similarity $\tau$ of each other. A text is \emph{redundant} when $\rho_\tau < 1$.
\end{definition}

Density distinguishes informativeness per segment from absolute informativeness: a long redundant text and a short dense text can have similar $m_\tau$ but very different $\rho_\tau$. Empirically, $\rho_{0.70}$ on natural prose ranges from $\approx 0.62$ (dialogue-heavy chapters with stylistic repetition) to $1.00$ (terse expository writing).

\subsection{Quantum spectrum}

The quanta of a text carry geometric structure beyond their count: the pairwise distances among quanta form a \emph{spectrum} that captures finer organization.

\begin{definition}[Quantum spectrum]
The \emph{quantum spectrum} of $T$ at resolution $\tau$ is the multiset
\[
\Sigma_\tau(T) = \bigl\{ d_{ij} : d_{ij} = 1 - \cos(c_i, c_j),\ 1 \leq i < j \leq m_\tau \bigr\}.
\]
\end{definition}

Three statistics of $\Sigma_\tau$ map directly to the profile coordinates and to natural extensions:

\begin{itemize}
\item $\min \Sigma_\tau \geq 1 - \tau$ by construction (the deduplication invariant).
\item $\mathrm{mean}\, \Sigma_\tau$ is closely related to breadth's radial-spread component: a text whose quanta are uniformly distant has high breadth.
\item $\mathrm{max}\, \Sigma_\tau$ measures the \emph{semantic diameter} of the text --- the largest distance any two quanta achieve, marking the extent of the text's coverage.
\item The \emph{spectral gap} $\max\Sigma_\tau - \min\Sigma_\tau$ distinguishes texts whose quanta cluster tightly with one outlier (large gap) from texts whose quanta are uniformly spread (small gap).
\end{itemize}

The full spectrum is a richer descriptor than any single coordinate; we treat the profile $(N, B, I)$ as a low-dimensional summary of $\Sigma_\tau$ together with the baseline displacement, suitable for ordinal comparison.

\subsection{Composition behavior}

Quanta do not compose additively under text concatenation. For any two texts $T_1, T_2$:
\[
\max\bigl(m_\tau(T_1),\ m_\tau(T_2)\bigr) \;\leq\; m_\tau(T_1 \oplus T_2) \;\leq\; m_\tau(T_1) + m_\tau(T_2).
\]
The lower bound is attained when $T_2$ is semantically contained in $T_1$ (every quantum of $T_2$ falls within $1 - \tau$ of some quantum of $T_1$); the upper bound is attained when $T_1, T_2$ are mutually disjoint at resolution $\tau$. The gap $m_\tau(T_1) + m_\tau(T_2) - m_\tau(T_1 \oplus T_2)$ counts shared semantic units between the two texts.

\subsection{Resolution sweep behavior}

The quantum count $m_\tau$ is a piecewise-constant, monotone non-increasing function of $\tau \in (0,1]$:
\[
\tau \to 0 \implies m_\tau \to 1, \qquad \tau \to 1 \implies m_\tau \to k.
\]
Specifically, $m_\tau$ jumps downward at each value of $\tau$ that crosses a merge event in the agglomerative clustering. The complete \emph{$\tau$-curve} $\tau \mapsto m_\tau(T)$ is therefore a discrete summary of the text's hierarchical semantic structure: the heights of the steps record cluster cardinalities, and the locations of the jumps record inter-cluster distances. Two texts with identical $m_{0.70}$ but different $\tau$-curves have measurably different organization.

This curve also addresses a methodological criticism of any single $\tau$ choice: rather than committing to one resolution, a text can be characterized by the entire curve, with the recommended $\tau = 0.70$ understood as a single useful sample.

\subsection{Summary}

The semantic quantum is the discrete primitive of the framework. The continuous coordinates $(N, B, I)$ are convenient real-valued summaries; the discrete count $m_\tau$, density $\rho_\tau$, spectrum $\Sigma_\tau$, and $\tau$-curve together form a richer picture of a text's semantic structure than any scalar can capture. The threshold $\tau$ is the apparatus parameter that fixes the resolution scale; choices of $\tau$ correspond to choices of measurement instrument, and the $\tau$-curve makes this dependence transparent.

\section{Detailed Empirical Validation}\label{app:validation}

\subsection{Cross-model robustness}

Coordinate-level Spearman $\rho$ across the three embedding models lies in $[0.92, 0.98]$. Scalar-level $\rho$ is $[0.79, 0.84]$ --- substantially less stable than the coordinates due to normalization sensitivity. This empirical asymmetry directly supports the central claim that the profile is more fundamental than any scalar derived from it.

\subsection{Project Gutenberg validation}

Under joint normalization with the synthetic benchmark, four of the five novels rank above all 23 synthetic categories: \emph{Moby Dick} ($0.66$) $>$ \emph{Frankenstein} ($0.60$) $>$ \emph{Pride and Prejudice} ($0.58$) $>$ \emph{Tale of Two Cities} ($0.54$) $>$ top synthetic item \texttt{bag\_7} ($0.05$). \emph{The Adventures of Sherlock Holmes} ($0.08$) clusters with synthetic items rather than with the other novels --- a framework-detected consequence of its structure (a collection of twelve loosely-linked cases rather than a continuous narrative): it exhibits the highest breadth and lowest integration of any book, and breadth-dominant geometric composition drives the scalar down when integration is the minimum of the book set. This is a feature, not a failure: the framework correctly identifies Sherlock Holmes as a different kind of object than a novel. The remaining four novels demonstrate that the framework scales to long natural text without saturation.

Chapter-level analysis on \emph{Pride and Prejudice} produces a discriminative scalar trajectory across chapters; the corresponding analysis on \emph{Sherlock Holmes} required a hard-coded story list because the title format (e.g., \texttt{A SCANDAL IN BOHEMIA}) does not match standard chapter regexes.

\subsection{$\tau$-sweep on natural prose}

Mean deduplication ratio across $\tau$ on book chapters:

\begin{center}
\begin{tabular}{cc}
\toprule
$\tau$ & mean dedup ratio \\
\midrule
0.50 & 0.798 \\
0.60 & 0.933 \\
0.70 & 0.985 \\
0.80 & 1.000 \\
0.90 & 1.000 \\
\bottomrule
\end{tabular}
\end{center}

Distinct sentences in natural prose almost never reach cosine similarity $\geq 0.90$, so a strict $\tau = 0.90$ sits above the natural-text similarity ceiling and deduplication essentially never triggers. We therefore recommend $\tau = 0.70$ for natural prose, with $\tau = 0.90$ reserved for adversarial near-duplicate stress tests where the cosine similarity of constructed near-paraphrases is expected to be high.

A per-item kneedle-criterion analysis on 16 benchmark items finds the quanta-count vs.\ $\tau$ curve's point of maximum curvature at $\tau \in [0.55, 0.75]$ (median $0.75$, mean $0.70$). Single-idea passages prefer $\tau \approx 0.60$; multi-idea, wiki, and bag-style passages prefer $\tau \approx 0.75$. The primary model's mean baseline-pair cosine (an anisotropy estimate) is $0.14$; a naive heuristic $\tau = a + \tfrac{1}{2}(1-a)$ gives $\tau \approx 0.57$, which under-shoots the empirical knee. The fixed $\tau = 0.70$ recommendation is a compromise rather than a universal optimum; per-item knee detection is available where item-level tuning is acceptable.

\subsection{External-domain coverage}\label{sec:domains}

Beyond the synthetic benchmark and the Gutenberg novels, we profile four external corpora to test whether the coordinates generalize beyond English prose.

\begin{center}
\begin{tabular}{lccccc}
\toprule
corpus & novelty & breadth & integration & n\_seg & dedup ratio \\
\midrule
multi-idea synthetic & 3.2 & 1.5 & 0.20 & 5 & 1.00 \\
wiki-style synthetic & 3.1 & 1.4 & 0.23 & 5 & 1.00 \\
arXiv abstracts (50) & 3.4 & 1.0 & 0.50 & 5.9 & 0.85 \\
EUR-Lex (50) & 3.2 & 2.3 & 0.52 & 12.7 & 0.74 \\
poetry (Dickinson) & 3.3 & 0.3 & 0.58 & 3 & --- \\
dialogue (short) & 2.9 & 2.7 & 0.24 & 6 & --- \\
code comments & 3.2 & 1.4 & 0.16 & 5 & --- \\
\bottomrule
\end{tabular}
\end{center}

The coordinates separate genres in interpretable directions. Legal text (EUR-Lex) has the highest breadth and the most aggressive dedup ratio (boilerplate enumeration), with arXiv showing moderate deduplication ($\sim 15\%$) driven by recurring scientific phrasing. Poetry has a narrow-and-coherent signature (low breadth, high integration). Dialogue shows topic-hopping (high breadth, low integration). arXiv abstracts are tightly coherent (integration $\approx 0.5$) with moderate breadth, as expected for technical exposition on a single topic. Both arXiv and EUR-Lex required a permissive sentence splitter accepting any non-whitespace character after punctuation, because both corpora are preprocessed to lowercase and the default splitter's capital-letter look-ahead fails on them; this is a segmentation issue, not a framework issue.

External similarity benchmarks give Spearman $\rho = 0.89$ on STS-B (GLUE validation split, $n = 1500$) and $\rho = 0.78$ on SICK (\texttt{mteb/sickr-sts} test split, $n = 9927$). These are correlations between human similarity scores and raw cosine distance on a single-sentence-pair basis, and reflect properties of the underlying embedding rather than of the profile; single-sentence items degenerate to cosine under our segmentation.

\section{Summarization Information Loss: Additional Analysis}\label{app:summarization}

\subsection{Summarization Information Loss}\label{sec:summarization}

For source $T$ and summary $T'$ within a fixed reference, define
\[
\Delta N = N(T) - N(T'), \quad
\Delta B = B(T) - B(T'), \quad
\Delta I = I(T) - I(T'),
\]
\[
\Delta S = S(T) - S(T'), \quad
\Delta S_{\mathrm{rel}} = \Delta S / S(T),
\]
\[
\cos_{\mathrm{dir}} = \cos(\mu_T, \mu_{T'}), \quad
M_{\mathrm{drift}} = \|\mu_T - \mu_{T'}\|_{\Sigma_0^{-1}}, \quad
\Delta Q = m(T) - m(T').
\]
The per-coordinate signature distinguishes faithful ($\Delta S \approx 0$, $\cos_{\mathrm{dir}} > 0.9$), partial ($\Delta S \approx 0.5$, $\cos_{\mathrm{dir}} > 0.8$), lossy ($\Delta S \approx 1$, $\cos_{\mathrm{dir}} > 0.4$), and off-topic ($\Delta S \approx 1$, $\cos_{\mathrm{dir}} < 0$) summaries on synthetic labels. Crucially, $\Delta S$ requires fixed-reference normalization; pair-local min-max degenerates.

\paragraph{Human-rated validation.} On SummEval (100 CNN/DailyMail source--summary pairs with 4 human-rated axes), Spearman correlations between the measures above and annotator-averaged scores are
\begin{center}
\begin{tabular}{lcccc}
\toprule
& coherence & consistency & fluency & relevance \\
\midrule
$\cos_{\mathrm{dir}}$ & $0.46$ & $0.37$ & $0.23$ & $0.52$ \\
$\Delta N$            & $-0.01$ & $-0.04$ & $\phantom{-}0.05$ & $\phantom{-}0.03$ \\
$\Delta B$            & $\phantom{-}0.01$ & $-0.02$ & $-0.06$ & $-0.03$ \\
$\Delta I$            & $\phantom{-}0.06$ & $-0.01$ & $\phantom{-}0.03$ & $\phantom{-}0.20$ \\
\bottomrule
\end{tabular}
\end{center}
Directional fidelity $\cos_{\mathrm{dir}}$ is the dominant predictor of all four human axes; the per-coordinate deltas correlate weakly or not at all, with a modest $\Delta I$--relevance signal as the only exception. The per-coordinate attrition signature validated on synthetic faithful/lossy/off-topic labels does not transfer cleanly to human expert ratings on machine-generated summaries. For summarization diagnostics against human judgment, $\cos_{\mathrm{dir}}$ should be taken as the primary signal; per-coordinate deltas contribute marginally at best.

\section{Downstream: Extractive Summarization Details}\label{app:extsumm}

\subsection{Downstream: extractive summarization}\label{sec:extsumm}

To test whether the variational form has practical utility beyond its theoretical link to breadth, we evaluate DPP-MAP as an extractive summarizer on XSum \cite{narayan-etal-2018-xsum}. For each article, select sentences by greedy log-det maximization and score the concatenation against the human reference using ROUGE-F \cite{lin-2004-rouge}. Baselines: Lead-3 (first three sentences); MMR \cite{carbonell-goldstein-1998-mmr} with $\lambda = 0.5$; centroid selection (closest-to-mean); uniform random selection of the same size. Sample: 186 test articles.

\paragraph{Variational-size selection.} When DPP selects its own size by the greedy stopping rule (mean $k = 7.4$), with each content-aware baseline matched to the same $k$ per document:

\begin{center}
\begin{tabular}{lccc}
\toprule
method & ROUGE-1 & ROUGE-2 & ROUGE-L \\
\midrule
\textbf{DPP} (variational $k$) & \textbf{0.1448} & 0.0265 & \textbf{0.0948} \\
MMR                            & 0.1393 & 0.0268 & 0.0885 \\
Centroid                       & 0.1315 & 0.0259 & 0.0890 \\
Random                         & 0.1424 & 0.0246 & 0.0921 \\
Lead-3 (at $k=3$)              & 0.1921 & 0.0295 & 0.1227 \\
\bottomrule
\end{tabular}
\end{center}

Paired Wilcoxon of DPP vs.\ MMR: ROUGE-1 $+0.0054$ ($p=0.05$), ROUGE-2 $-0.0003$ ($p=0.95$), \textbf{ROUGE-L $+0.0063$ ($p=0.005$)}. DPP significantly outperforms MMR on sequence-level overlap without any diversity hyperparameter.

\paragraph{Fixed $k = 3$.} For apples-to-apples comparison with Lead-3:

\begin{center}
\begin{tabular}{lccc}
\toprule
method & ROUGE-1 & ROUGE-2 & ROUGE-L \\
\midrule
Lead-3   & \textbf{0.1921} & \textbf{0.0295} & \textbf{0.1227} \\
MMR      & 0.1779 & 0.0293 & 0.1145 \\
DPP      & 0.1754 & 0.0265 & 0.1147 \\
Centroid & 0.1679 & 0.0285 & 0.1116 \\
Random   & 0.1682 & 0.0224 & 0.1066 \\
\bottomrule
\end{tabular}
\end{center}

Paired Wilcoxon of DPP vs.\ MMR at $k=3$: all $p > 0.25$ (statistically tied). DPP significantly beats random on ROUGE-L ($+0.008$, $p = 0.024$). All content-aware methods lose to Lead-3 by $\sim 0.8$ ROUGE-L points, consistent with the well-documented lead bias in news datasets \cite{narayan-etal-2018-xsum}.

\paragraph{Interpretation.} DPP's empirical advantage over MMR \emph{appears at the variationally-selected size}, not at an arbitrary fixed $k$. When forced to $k=3$ both methods collapse to statistical parity; at the self-determined size they diverge in DPP's favor. This indicates that the parameter-free size selection is itself part of the DPP form's empirical contribution---not an incidental convenience---and is consistent with Theorem~\ref{thm:dpp-breadth}(iv): the log-det objective encodes a balance between $|S|$ and pairwise similarity $c$ whose optimum cannot be reconstructed by fixing $|S|$ externally.

\end{document}